\documentclass[]{spie}  

\usepackage{amsmath,amsfonts,amssymb}
\usepackage{graphicx}
\usepackage{physics}
\usepackage{algorithm}
\usepackage{algpseudocode}
\usepackage{mathtools}
\usepackage{siunitx}

\usepackage[colorlinks=true, allcolors=blue]{hyperref}

\newcommand\numberthis{\addtocounter{equation}{1}\tag{\theequation}}
\newcommand{\R}{\mathbb{R}}
\newcommand{\set}[1]{\{ #1 \}}
\DeclarePairedDelimiterX{\dprod}[2]{\langle}{\rangle}{#1, #2}

\DeclareMathOperator*{\argmin}{\arg\! \min}

\DeclarePairedDelimiter{\floor}{\lfloor}{\rfloor}

\title{Characterization of point-source transient events with a rolling-shutter compressed sensing system}

\author[1]{Frank Qiu}
\author[1]{Joshua Michalenko}
\author[1]{Lilian K. Casias}
\author[1]{Cameron J. Radosevich}
\author[1]{Jon Slater}
\author[1]{Eric A. Shields}
\affil[1]{Sandia National Laboratories, Albuquerque, New Mexico 87123, USA}

\authorinfo{Further author information: (Send correspondence to Frank Qiu)\\Frank Qiu.: E-mail: fyqiu@sandia.gov}

\begin{document} 
\maketitle

\begin{abstract}
Point-source transient events (PSTEs) - optical events that are both extremely fast and extremely small - pose several challenges to an imaging system. Due to their speed, accurately characterizing such events often requires detectors with very high frame rates. Due to their size, accurately detecting such events requires maintaining coverage over an extended field-of-view, often through the use of imaging focal plane arrays (FPA) with a global shutter readout. Traditional imaging systems that meet these requirements are costly in terms of price, size, weight, power consumption, and data bandwidth, and there is a need for cheaper solutions with adequate temporal and spatial coverage. To address these issues, we develop a novel compressed sensing algorithm adapted to the rolling shutter readout of an imaging system. This approach enables reconstruction of a PSTE signature at the sampling rate of the rolling shutter, offering a 1-2 order of magnitude temporal speedup and a proportional reduction in data bandwidth. We present empirical results demonstrating accurate recovery of PSTEs using measurements that are spatially undersampled by a factor of 25, and our simulations show that, relative to other compressed sensing algorithms, our algorithm is both faster and yields higher quality reconstructions. We also present theoretical results characterizing our algorithm and corroborating simulations. The potential impact of our work includes the development of much faster, cheaper sensor solutions for PSTE detection and characterization.
\end{abstract}


\section{Introduction}
The detection and characterization of point-source transient-events  (PSTEs), extremely fast and small optical events, present several challenges to the data-processing pipeline of a traditional imaging system. Firstly, their speed necessitates a high temporal sampling rate for accurate characterization. Secondly, their limited spatial size often requires maintaining coverage over an extended field-of-view for reliable detection, necessitating the use of imaging focal plane arrays (FPA) with a global shutter readout. Current state-of-the-art commercial cameras with sufficient speed and spatial coverage are expensive, power hungry, heavy, unable to operate continuously, and generate high-bandwidth  data transmissions. These characteristics make such cameras difficult to operate and limit their practical effectiveness in a variety of settings.

To address these challenges, we leverage modern advancements in computational imagery and optical imaging techniques to enhance the capabilities of a traditional imaging system. On the computational side, the field of compressed sensing (CS) \cite{CandesWakinIntro} provides robust techniques for the recovery of sparse signals from far fewer samples than traditional Shannon-Nyquist sampling theory would dictate. Combined with recent advances in optical sampling techniques, CS algorithms have demonstrated the ability to create video products with frame rates orders of magnitude higher than what a sensor can natively support \cite{VidFromStills, DiffuserCamAntipa, Weinberg100k, BoominathanRecentAdv, WillettDiff}. This also simultaneously reduces the data bandwidth by the same factor: in our simulations, we used a $\SI{40}{\hertz}$ global shutter rate camera to image at $\SI{1000}{\hertz}$, offering a data reduction factor of 25 relative to a camera operating natively at $\SI{1000}{\hertz}$.

In  this paper, we present work on algorithms adapted to the rolling shutter readout of an imaging system\cite{VidFromStills,DiffuserCamAntipa, Weinberg100k}. Section \ref{sec:OpticalSystem} provides an overview of our imaging system, and Section \ref{sec:AlgoOverview} introduces the algorithms used to recover the original signal. Section \ref{sec:Theory} provides a theoretical characterization of both our imaging system and algorithms, giving conditions under which successful signal recovery can happen. Section \ref{sec:Experiments} contains simulation results that compare our algorithm to other benchmark CS algorithms, validate our theoretical results, and demonstrate the limits of signal recovery under our signal processing pipeline. Section \ref{sec:PractConsid} addresses challenges to signal recovery using our current imaging system and discusses two potential hardware solutions.

\section{Imaging System Overview} \label{sec:OpticalSystem}
A traditional imaging system is typically comprised of three main components:  1) the optical elements used to form a focused image at the focal plane, 2) an array of photo-sensitive pixels called the focal plane array (FPA), 3) the readout electronics used to measure and transfer the signal from the FPA. The optics of such systems are optimized for imaging performance, so that light originating from a point-source object is localized to a small number of adjacent pixels on the FPA. A rolling shutter readout is often used, where each line of the FPA is exposed, sampled, and read-out sequentially. As such, pixels corresponding to a point-source event are only sampled on the order of the global frame rate - the rate of exposing, sampling, and reading out all lines of the FPA.  In contrast, the rolling shutter rate - the rate of recording a new line - is orders of magnitude faster,  and several works leverage the rolling shutter readout to increase the sampling rate\cite{VidFromStills,DiffuserCamAntipa, Weinberg100k}. To combat spatial under-sampling inherent to the rolling shutter readout, these approaches introduce a diffuser-like element within the optical train to spread light over the entire FPA, and the original signal is recovered by applying a reconstruction algorithm to the diffused signal's rolling shutter readout.


Our simulated imaging system in Sections \ref{sec:Experiments} and \ref{sec:PractConsid} takes this approach. The input signal is passed through a phase diffuser, and a rolling shutter sampler is applied to the diffused signal to generate measurements. Figure \ref{fig:OpticalSystem} visualizes the diffuser's point spread function (PSF) as well as the rolling shutter schedule, and Figure \ref{fig:MeasurementSeq} visualizes the various stages of signal processing in a rolling shutter imaging system. While the shutter reads a single line at the rolling shutter rate, we can choose to sample at slower frequencies by integrating over an appropriate time window. For example, if the rolling shutter rate is $\SI{5000}{\hertz}$ and we wish to sample at $\SI{1000}{\hertz}$, we would integrate the readout over a $\SI{1}{\milli\second}$ window and hence sample five lines at a time. Accordingly, the rolling shutter schedule of Figure \ref{fig:OpticalSystem} shows a rolling shutter that samples five lines at a time. The effect of choosing an appropriate integration window (and hence number of lines per sample) is discussed theoretically in Section \ref{sec:Theory} and practically in Section \ref{sec:Experiments}.

\begin{figure} [hb]
    \begin{center}
    \begin{tabular}{c} 
    \includegraphics[width = 16cm]{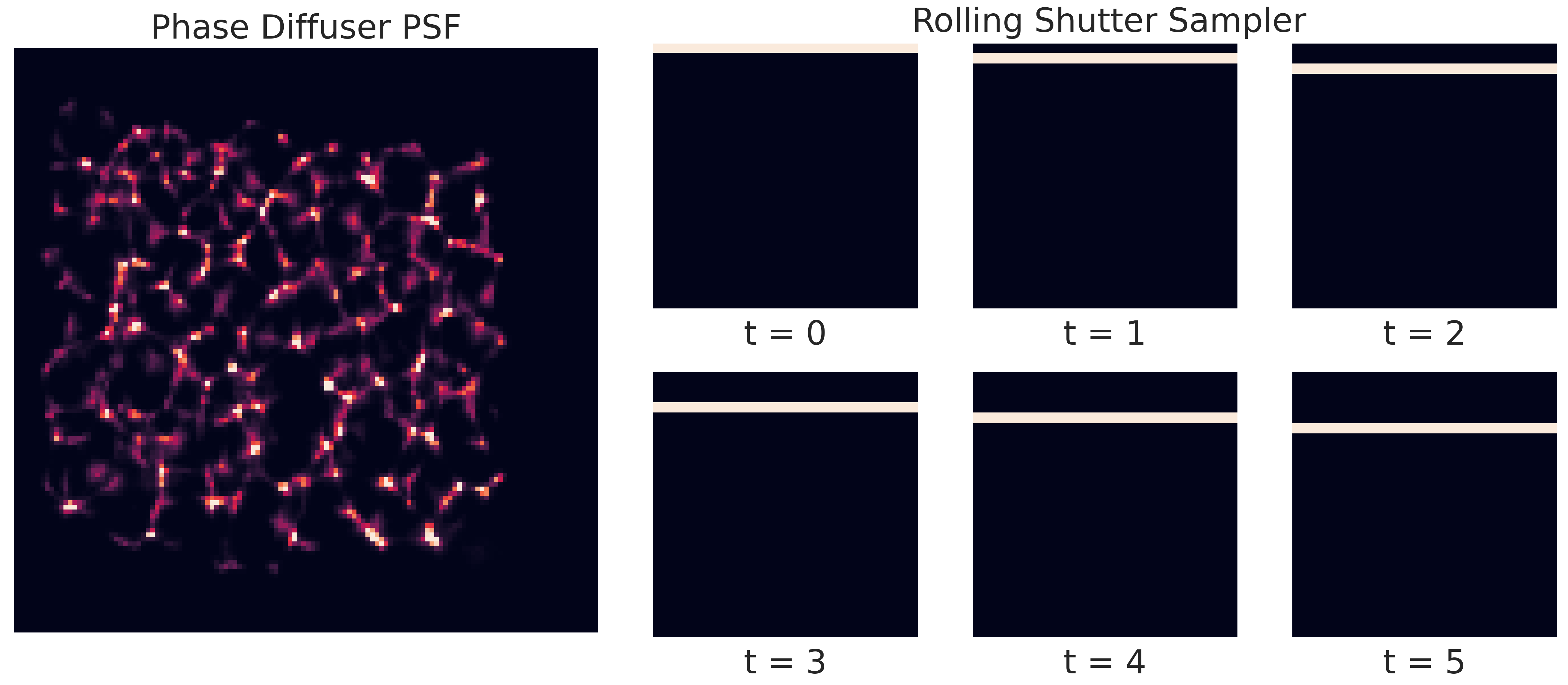}
    \end{tabular}
    \end{center}
    \caption{\label{fig:OpticalSystem}
    Overview of the simulated rolling shutter system in Sections \ref{sec:Experiments} and \ref{sec:PractConsid}. On the left is the diffuser's PSF, recorded from a physical phase diffuser \cite{VidFromStills}. On the right, we visualize the sampling schedule of the rolling shutter over time, with the white lines representing the sampled lines. Note that the diffuser's PSF spreads a signal over most of the FPA, which combats the spatial undersampling inherent to the rolling shutter readout.}
\end{figure} 

\begin{figure} [ht]
    \begin{center}
    \begin{tabular}{c} 
    \includegraphics[width = 17cm]{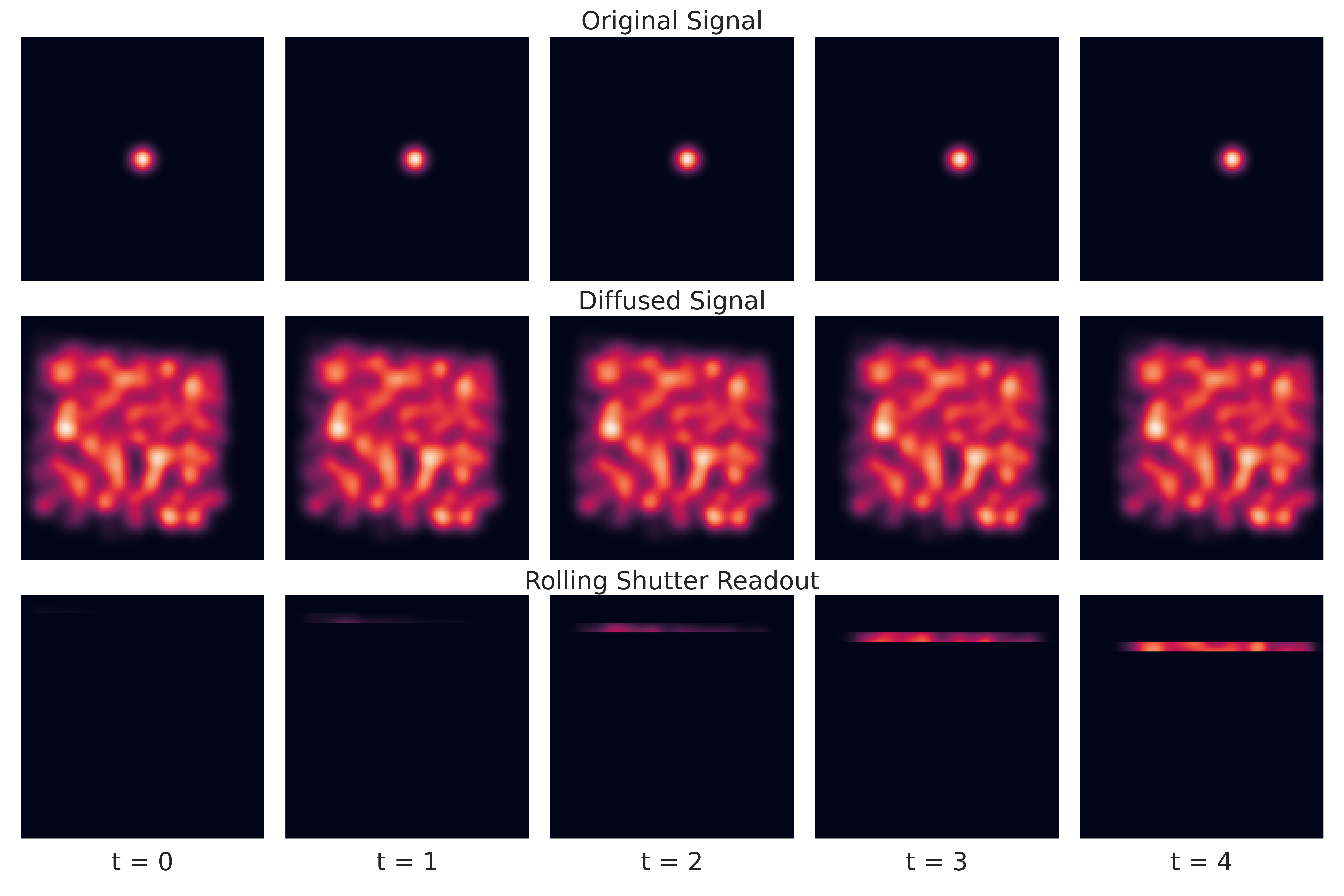}
    \end{tabular}
    \end{center}
    \caption{\label{fig:MeasurementSeq}
    Visualization of the various stages of processing in our rolling shutter imaging system. The top row shows the raw input signal of a light source moving horizontally. The middle row shows the result of passing the input through the phase diffuser. The bottom row shows measurements recorded from the rolling shutter readout. From these measurements, we recover the original signal of the top row by applying our reconstruction algorithms in Section \ref{sec:AlgoOverview}.}
\end{figure}

\section{Algorithms Overview}\label{sec:AlgoOverview}
At each time point, our rolling shutter imaging system samples $M$ of the $P$ total pixels in the FPA. Accordingly, our measurement model at time $t$ can be written as:
\begin{equation}\label{eqn:MeasurementModel}
    y^{(t)} = A^{(t)} x^{(t)} + z^{(t)}
\end{equation}
where $y^{(t)} \in \R^M$ is our measurement, $A^{(t)} \in \R^{M \times P}$ is a sensing matrix representing our imaging system, $x^{(t)} \in \R^{P}$ is the input to the system, and $z^{(t)} \in \R^M$ is any exogenous noise. The sensing matrix $A^{(t)}$ can be further decomposed into:
\begin{gather}
    A^{(t)} = P_r^{(t)} H \label{eqn:SensingMatrixDecomp}\\
    H v \equiv  \zeta \ast v \label{eqn:CircmatrixDecomp}
\end{gather}
where $P_r^{(t)}$ is a projection matrix representing the rolling shutter at time $t$ and $H$ is a circulant matrix representing convolution with the PSF $\zeta$. Note that the sensing matrix $A^{(t)}$ is time-dependent, since the rolling shutter will sample different line subsets as it sweeps over the FPA.

\subsection{Algorithm \ref{alg:FISTA_Diff}: FISTA with Differences}
Let $x^* \in \R^{P \times T}$ denote the input movie of $T$ images and $y \in \R^{M \times T}$ the corresponding sequence of measurements. Our goal is recover $x^*$ from only the measurements $y$. To do this, we adopt a compressed sensing approach and reconstruct by minimizing the following objective:
\begin{equation}\label{eqn:DiffObj_2}
    \hat{x} = \argmin_x \sum_{t = 0}^{T-1} \frac{1}{2}\| A^{(t)} x^{(t)} - y^{(t)}\|^2_2 + \lambda \| x^{(t)} - x^{(t -1)} \|_1
\end{equation}
where we set $x^{(-1)} \equiv 0$ and $\lambda > 0$ is a regularization parameter. In contrast to a standard compressed sensing objective, we penalize the $\ell^1$ norm of the frame-by-frame differences rather than the frames themselves, an approach previously used in another coded aperture system\cite{WillettDiff}. Defining $d^{(t)} \equiv x^{(t)} - x^{(t-1)}$, there is a bijection between a movie $x \in \R^{P \times T}$ and its corresponding sequence of differences $d \in \R^{P \times T}$:
\[
x^{(t)} = \sum_{s \leq t} d^{(s)}
\] 
Hence, we can instead recover the differences of $x^*$ by rewriting equation \ref{eqn:DiffObj_2}:
\begin{equation}\label{eqn:DiffObj}
    \hat{d} = \argmin_d \sum_{t = 0}^{T-1} \frac{1}{2}\| A^{(t)} (\sum_{s \leq t} d^{(s)}) - y^{(t)}\|_2^2 + \lambda \| d^{(t)} \|_1
\end{equation}
The solutions of equations \ref{eqn:DiffObj_2} and \ref{eqn:DiffObj} are equivalent. 
However, working with differences yields simpler optimization equations, and so we use the differences formulation of equation \ref{eqn:DiffObj}.

Since our optimization problem is a $\ell^2 - \ell^1$ minimization problem, we use the Fast Iterative Shrinkage-Thresholding Algorithm\cite{FISTA} (FISTA) to solve equation \ref{eqn:DiffObj}. Let $\eta \in \R^{M \times T}$ denote the error between the predicted measurement sequence under $\hat{x}$ and the observed measurements $y$, with $\eta^{(t)} = A^{(t)}[\sum\limits_{s \leq t} d^{(s)}-x^{*(t)} ]$.  The FISTA update rule at iterate $m$ can be written as:
\begin{equation}\label{eqn:FISTA_update}
d_{m+1} = soft_\lambda(d_m - \tau H^T S^T \eta_m)
\end{equation}
where $\tau$ is the stepsize, $soft_\lambda$ is the soft-thresholding operator with parameter $\lambda$, and $S$ is an lower triangular matrix of 1's. Appendix \ref{app:OptDetails} contains a derivation of the FISTA update rule and a formula for calibrating the stepsize $\tau$.

\subsection{Algorithm \ref{alg:BlockedFISTA-D}: Blocked FISTA with Differences}
While we can reconstruct the entire movie at once, the sequential nature of movies suggests a divide-and-conquer approach. In the blocked version of the algorithm, we first partition the measurement sequence into blocks of equal length $B$. Having generated difference estimates for the $k^{th}$ block $\hat{d}_k$, the reconstruction of the final frame of block $k$ is the sum of all differences:
\[
\hat{x}^{(B-1)}_k = \sum_{s = 0}^{B-1} \hat{d}^{(s)}_k
\]
We then initialize the first frame of block $k + 1$ with this estimate:
\[
\hat{d}^{(0)}_{k+1} = \hat{x}^{(0)}_{k+1} \leftarrow \hat{x}^{(B-1)}_k
\]
We found that the blocked version is both faster and more accurate than the unblocked version of the previous section. From equation \ref{eqn:FISTA_update}, the time complexity of each optimization step scales with the size of the input sequence, so optimizing over a smaller block is faster than optimizing over the whole movie. As PSTEs are temporally transient, the algorithm converges extremely quickly on the many non-active blocks and concentrates computation on the few active blocks, resulting in an overall speedup. Similarly, the theoretical results of Section \ref{sec:ErrorBounds} suggest that reconstruction error grows with the length of the input, as errors can build up from summing the difference estimates. Blocking allows us to control this error buildup and maintain high quality reconstructions over movies of arbitrary length. Figure \ref{fig:algoDescription} contains pseudocode for both algorithm variants.

\begin{figure}
\begin{minipage}{0.46\textwidth}
\vspace{0pt}
\begin{algorithm}[H]
    \centering
    \caption{FISTA with Differences (FISTA-D)}\label{alg:FISTA_Diff}
    \begin{algorithmic}[1]
        \State \textbf{Input:} 
            \State \hspace{1em} measurements: $y \in \R^{M \times T}$ 
            \State \hspace{1em} initialization: $\hat{d} \in \R^{P \times T}$
        \State \textbf{Hyperparameters:} 
            \State \hspace{1em} stepsize: $\tau > 0$
            \State \hspace{1em} regularization parameter: $\lambda > 0$
        \State
        \While{not converged} 
        \For{$t \in {0,\cdots, T-1}$}
        \State $\eta^{(t)} \gets A^{(t)}[\sum\limits_{s \leq t} d^{(s)} - x^{*(t)} ]$
        \EndFor
        \State $\hat{d} \gets soft_\lambda(\hat{d} - \tau H^T S^T \eta)$
        \EndWhile
        \State
        \State \textbf{return} $\hat{d}$
    \end{algorithmic}
\end{algorithm}
\end{minipage}
\hfill
\begin{minipage}{0.46\textwidth}
\vspace{0pt}
\begin{algorithm}[H]
    \centering
    \caption{Blocked FISTA with Differences}\label{alg:BlockedFISTA-D}
    \begin{algorithmic}[1]
        \State \textbf{Input:} 
            \State \hspace{1em} measurements $y \in \R^{M \times T}$ 
        \State \textbf{Hyperparameters:} 
            \State \hspace{1em} stepsize: $\tau > 0$
            \State \hspace{1em} regularization parameter: $\lambda > 0$
            \State \hspace{1em} block length: $B \in \mathbb{Z}_+$
        \State
        \State \textbf{initialize:} $\hat{x}_{-1} \gets 0 \in \R^{P}$
        \For{$b \in \set{0,\cdots,  \floor{\frac{T}{B}}}$}
        \State $y_b \gets \textbf{cat}(y^{(bB)}, y^{(bB+1)},\cdots, y^{((b+1)B -1)})$
        \State $\hat{d}_b \gets \boldsymbol{0}$
        \State $\hat{d}_b^{\tiny (0)} \gets \hat{x}_{b-1}$
        \State $\hat{d}_b \gets$ FISTA-D($y_b, \hat{d}_b, \tau, \lambda$)
        \State $\hat{x}_{b} \gets \sum\limits_{s=0}^{B-1} \hat{d}_b^{(s)}$
        \EndFor
        \State $\hat{d} \gets$ \textbf{cat}($\hat{d_0},\cdots,\hat{d}_{ \floor{\frac{T}{B}}}$)
        \State
        \State \textbf{return} $\hat{d}$
    \end{algorithmic}
\end{algorithm}
\end{minipage}    
\vspace{1em}
\caption{\label{fig:algoDescription} Pseudocode for both unblocked and blocked versions of the differences algorithm. Both algorithms output the sequence of differences estimates $\hat{d}$. Summing the difference estimates generates the movie reconstruction $\hat{x}$.}
\label{fig:my_label}
\end{figure}

\section{Theoretical Characterization}\label{sec:Theory}
In this section, we give an overview of the relevant theoretical aspects of our imaging system and algorithm. CS algorithms exploit the fact that real world signals are often extremely compressible, allowing reconstruction of the original signal from drastically undersampled measurements. Such compressible signals $x$ are (near) sparse in some basis $\Phi$, meaning they can be (approximately) expressed as $x = \Phi \Tilde{x}$ with $\Tilde{x}$ having only a few non-zero entries. In order to guarantee accurate reconstruction, these algorithms generally require that the sensing matrix satisfy the Restricted Isometry Property\cite{CandesRIP} (RIP), meaning that it preserves distances between sparse vectors. More precisely, we say that a matrix $A$ is RIP of order $k$ if there exists some $\delta_k > 0$ such that:
\[
(1-\delta_k) \| x \|_2^2 \leq \| A x \|_2^2 \leq (1 + \delta_k) \| x \|_2^2
\]
for all $x$ with at most $k$ non-zero entries - such vectors are also called $k$-sparse. Given a vector $x \in \R^{P}$, let $[x]_k$ denote its best $k$-sparse approximation which drops all but the $k$ largest magnitude entries of $x$. The near $k$-sparsity of $x$ is then denoted as $\sigma_k(x) \equiv \| x - [x]_k \|_1$. By slight abuse of notation, for a movie $x \in \R^{P \times T}$ let $\sigma_k(x) \equiv \sum_{t=0}^{T-1} \sigma_k(x^{(t)})$ denote the sum of each frame's near $k$-sparsity.

\subsection{Rolling Shutter and the Restricted Isometry Property} \label{sec:RollingShutterRIP}
We provide conditions and sparsity levels under which the sensing matrices of our rolling shutter system are guaranteed to be RIP. From equations \ref{eqn:SensingMatrixDecomp} and \ref{eqn:CircmatrixDecomp}, our sensing matrix $A^{(t)}$ at time $t$ can be written as the composition of a projection and convolution:
\[
A^{(t)} x = P_r^{(t)} \zeta \ast x
\]
Provided the PSF satisfies certain assumptions, we may apply the following theorem to our sensing matrices.

\begin{theorem}[Theorem 4.1 \cite{krahmer2013suprema}] \label{thm:RIPpsf}
Let $\zeta \in \R^n$ be a random vector with independent, mean-zero, variance one, $S$-subgaussian entries. Let $A \in \R^{m \times n}$ be the matrix defined as $Ax \equiv \frac{1}{\sqrt{m}} P_m [\zeta \ast x]$ for some projection matrix $P_m: \R^n \rightarrow \R^m$. If, for $k \leq n$ and $\eta,\delta \in (0,1)$
\begin{equation*}
    m \geq c_S \delta^{-2} k \max\set{\log^2(k) \log^2(n), \log(\eta^{-1})}
\end{equation*}
then with probability at least $1 - \eta$, the restricted isometry constant of the matrix $A$ satisfies $\delta_k \leq \delta$. The constant $c_S$ depends only on the subgaussian parameter $S$.
\end{theorem}

Taking a union bound over all unique sensing matrices $A^{(t)}$ (as the rolling shutter is periodic), Theorem \ref{thm:RIPpsf} establishes a linear relationship between the maximum support size $k$ of each frame and the number of lines sampled, up to a squared log factor. PSTEs are spatially limited with extremely small $k$, so we can sample a few lines at a time while still guaranteeing that our sensing matrices are uniformly RIP with constant $\delta_k$. The result requires that our PSF satisfy certain randomness conditions and have non-negative values, the latter of which is impossible to realize in a physical system. However, empirically we found that this result provides good predictions (see Sections \ref{sec:Experiments} and \ref{sec:PractConsid}), and mild violations of the theorem's assumptions do not affect its predictive power. From the discussion in Section \ref{sec:OpticalSystem}, the number of lines per sample is a function of the system's native rolling shutter rate and our choice of a time integration window. In light of this, Theorem \ref{thm:RIPpsf} can also be interpreted as requiring the size of the signal's spatial support to scale roughly linearly with the rolling shutter rate.

\subsection{Reconstruction Error Bounds} \label{sec:ErrorBounds}
A typical error bound in compressed sensing controls the reconstruction error by the near-sparsity of the underlying signal \cite{CandesTao}. However, some CS algorithms alternatively use a total variation (TV) regularizer, and in certain settings one can establish error bounds that scale with the near-sparsity of the total variation gradient \cite{KrahmerTvnorm, NeedellWard, NeedellWard2, PoonFourier}. Analogously, we establish a result bounding the average frame-wise reconstruction error by the near sparsity of the signal's time gradient plus a boundary condition. Sampled at sufficiently high rates, the frame-by-frame change in a signal might be very sparse even if the frames themselves are not sparse. This theoretical result complements the primary hardware advantage of our rolling shutter system, which drastically speeds up the sampling rate relative to a global shutter system.

While our algorithm recovers the difference estimates by solving equation \ref{eqn:DiffObj}, this is equivalent to directly solving for the reconstruction estimates via equation \ref{eqn:DiffObj_2}. Given a movie $x \in \R^{P \times T}$, let $\nabla_t x \in \R^{P \times (T-1)}$ denote its vector of forward time differences with $\nabla_t x^{(t)} = x^{(t+1)} - x^{(t)}$. Then, the objective of equation \ref{eqn:DiffObj_2} is the Lagrangian\cite{BoydOpt} of the following constrained optimization problem:
\begin{alignat*}{2}
    &\text{\normalfont minimize}  &  \| x^{(0)} \|_1 + ||\nabla_t x||_1\\
    &\text{\normalfont subject to} \quad & \quad \forall t, \ ||A^{(t)} x^{(t)}-  y^{(t)}||_2 \leq  \epsilon
\end{alignat*}
where there is a one-to-one relationship between $\epsilon$ and the regularization parameter $\lambda$. Our reconstruction $\hat{x}$ is the solution to the above optimization problem, and we have the following result.

\begin{theorem}\label{thm:avgErrDiff}
Let $x^* \in \R^{P \times T}$ be a movie of length $T$ and $\nabla_t x^* \in \R^{P \times (T-1)}$ denote its corresponding sequence of forward time differences. Given sensing matrices $A^{(t)}$ that are uniformly RIP with $\delta_{2k} < (4T)^{-1}$, suppose we have the measurements $y^{(t)} = A^{(t)}x^{*(t)} + z^{(t)}$ with $\norm{z^{(t)}}_2 \leq \epsilon$ for all times $t$. Let $\hat{x}$ be the solution to the following optimization problem:
\begin{alignat*}{2}
    &\text{\normalfont minimize}  &  \| x^{(0)} \|_1 + ||\nabla_t x||_1\\
    &\text{\normalfont subject to} \quad & \quad \forall t, \ ||A^{(t)} x^{(t)}-  y^{(t)}||_2 \leq  \epsilon
\end{alignat*}
For any pixel index set $I$ with $| I | \leq k$, let $\Bar{I} = \cup_t \Bar{I}(t)$ be the index set generated by repeating $I$ at each frame. For any such $\bar{I}$, we have:
\begin{equation}\label{bnd:SupportReconError}
\frac{1}{T} \sum_{t=0}^{T-1} \|x^{*(t)} - \hat{x}^{(t)}\|_2 \leq C \frac{\| x^{*(0)}_{I^c} \|_1 + \| \nabla_t x^*_{\bar{I}^c}\|_1}{\sqrt{k}} + C' T\epsilon
\end{equation}
where the constants $C, C'$ only depend on $\delta_{2k}$. If the $k$-sparse approximations of $\nabla_t x^*$ and $x^{(0)}$ share the same support, we have:
\begin{equation}\label{bnd:ReconError}
\frac{1}{T} \sum_{t=0}^{T-1} \|x^{*(t)} - \hat{x}^{(t)}\|_2 \leq C \frac{  \sigma_k(x^{*(0)}) + \sigma_k(\nabla_t x^*)}{\sqrt{k}} + C' T\epsilon
\end{equation}
\end{theorem}

Note that the $T$ parameter in Theorem \ref{thm:avgErrDiff} refers to the length of the input measurement sequence. This is significant because in the blocked version of our algorithm (Algorithm \ref{alg:BlockedFISTA-D}), $T = B$ where $B$ is the block length. Thus, the blocked version prevents the second error term from scaling with the total length of the movie, implying that accurate reconstruction can happen for movies of arbitrary length provided the time gradient is sparse. Similar concerns over the scaling of the RIP constant $\delta_{2k}$ with $T$ are alleviated in the blocked algorithm. Finally, the bound of equation \ref{bnd:ReconError} assumes that the time gradient concentrates on a constant set of $k$ pixels - provided our block length $B$ is small enough, this is a mild assumption due to the time continuity of signals. 

The above result allows us to characterize how different properties of our imaging system and signal affect the reconstruction error, and we validate these predictions in Section \ref{sec:ExpValidTheory}. Increasing the number of lines sampled while fixing the sampling frequency will increase the maximum support size $k$ by the discussion in Section \ref{sec:RollingShutterRIP}. This decreases the near-sparsity term $\sigma_k(x^{*(0)}) + \sigma_k(\nabla_t x^*)$ and hence gives a smaller reconstruction error. From our discussion in Section \ref{sec:RollingShutterRIP} and the results of Section \ref{sec:ExpValidTheory}, this amounts to choosing an appropriate integration window of the rolling shutter readout so the number of lines per sample is tuned to the size of the signal's spatial support. Similarly, increasing the sampling frequency or equivalently slowing the signal will shrink the overall magnitude of the time gradient and also give a smaller reconstruction error. Finally, the result predicts that the average reconstruction error should scale linearly with the root power of exogenous noise.

\section{Simulation  Results}\label{sec:Experiments}
We simulated a rolling-shutter imaging system and generated a static PSTE signature with varying intensity. The PSF of our simulated system was recorded from a physical phase diffuser\cite{VidFromStills}. Our simulated rolling shutter sampled 5 lines per sample at $\SI{1000}{\hertz}$ over a $128 \times 128$ FPA, which roughly corresponds to a $\SI{40}{\hertz}$ global shutter rate. The synthetic PSTE consisted of four sequential pulses of frequencies $\SI{15}{\hertz}$, $\SI{50}{\hertz}$, $\SI{100}{\hertz}$, and $\SI{400}{\hertz}$, with a spatial support of roughly 3 pixels in diameter and a duration of $\SI{300}{\milli\second}$. Figure \ref{fig:SimulatedPSTE} shows a cross-section of the PSTE in both time and space. 

\begin{figure} [ht]
    \begin{center}
    \begin{tabular}{c} 
    \includegraphics[width=15cm]{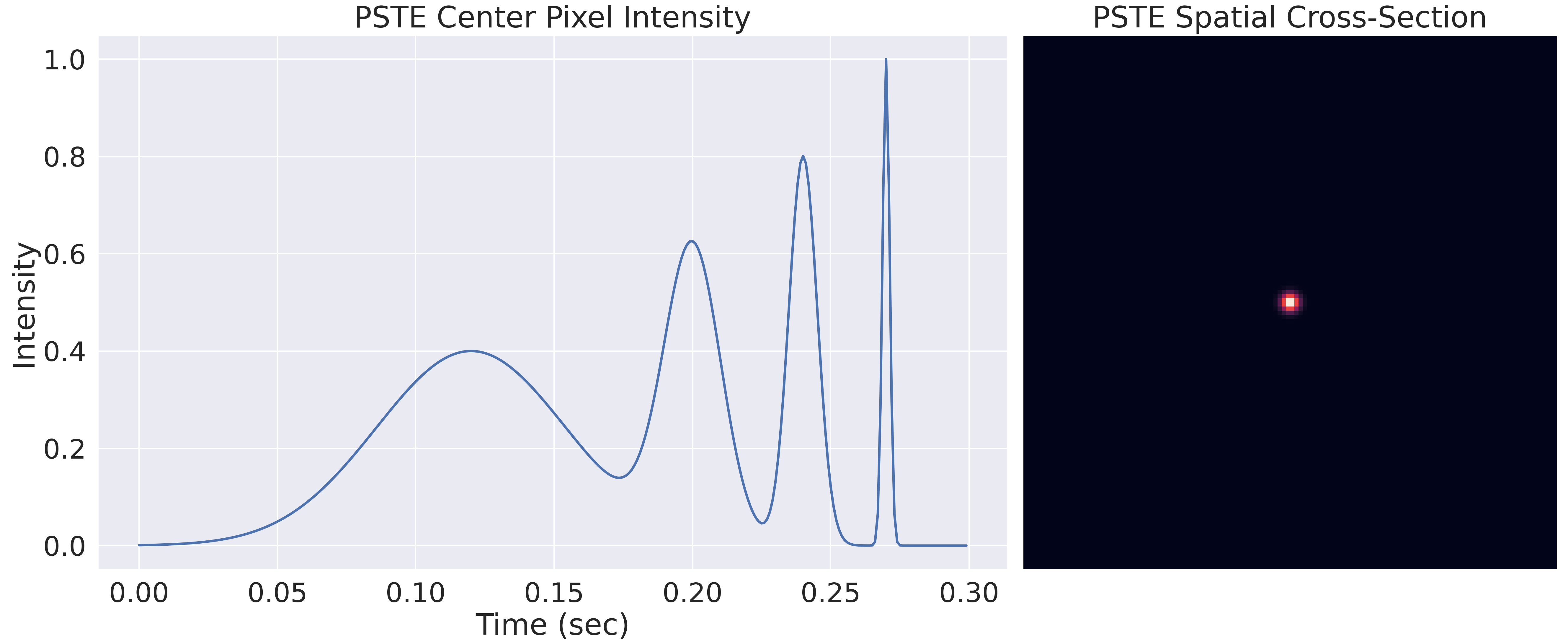}
    \end{tabular}
    \end{center}
    \caption{\label{fig:SimulatedPSTE}
On the left is a 1D signal consisting of four sequential pulses of roughly $\SI{15}{\hertz}$, $\SI{50}{\hertz}$, $\SI{100}{\hertz}$, and $\SI{400}{\hertz}$, with a duration of $\SI{300}{\milli\second}$. The PSTE was generated by multiplying this 1D signal with a $128 \times 128$ Gaussian, shown on the right, to generate a $128 \times 128 \times 300$ movie. Spatially, the PSTE is roughly 3 pixels in diameter. Our simulated rolling shutter reads 5 lines per sample at $\SI{1000}{\hertz}$, corresponding to a global shutter rate of roughly $\SI{40}{\hertz}$.}
\end{figure} 

\subsection{Comparison to Alternative  Reconstruction Algorithms}\label{sec:ComparisonAlgo}
Using the PSTE of Figure \ref{fig:SimulatedPSTE}, we compared the reconstruction quality of our blocked differences algorithm ($B=50$) against two alternative compressed sensing algorithms: one using a TV regularizer and one using a standard $\ell^1$ regularizer (see Appendix \ref{app:AltAlgos} for details). All three algorithms were optimized using FISTA, and they were allowed at most 10000 optimization steps with the same convergence threshold. The stepsizes were calibrated using similar calculations, and the regularization parameter was set to the same value $\lambda = .1$ with slight adjustments for the TV algorithm (see Appendices \ref{app:StepsizeCalc} and \ref{app:AltAlgos}). We compared reconstruction quality by looking at the center pixel where the PSTE is spatially localized, shown in Figure \ref{fig:ComparisonAlgo}. Notably, the reconstructions of the two alternative algorithms suffer from a periodic dropout artifact, which is not noticeably present in the reconstruction of our blocked differences algorithm. This is an artifact of the rolling shutter sampler, and we discuss this in more detail in Section \ref{sec:PractConsid}. In terms of speed, our differences algorithm ($\SI{31.1}{\second}$) was significantly faster than both the $\ell^1$ algorithm ($\SI{64.4}{\second}$) and the TV algorithm ($\SI{190.8}{\second}$). In summary, our algorithm was both faster and gave significantly higher quality reconstructions than the two alternative algorithms in our rolling shutter system.

\begin{figure} [ht]
    \begin{center}
    \begin{tabular}{c} 
    \includegraphics[width = 17cm]{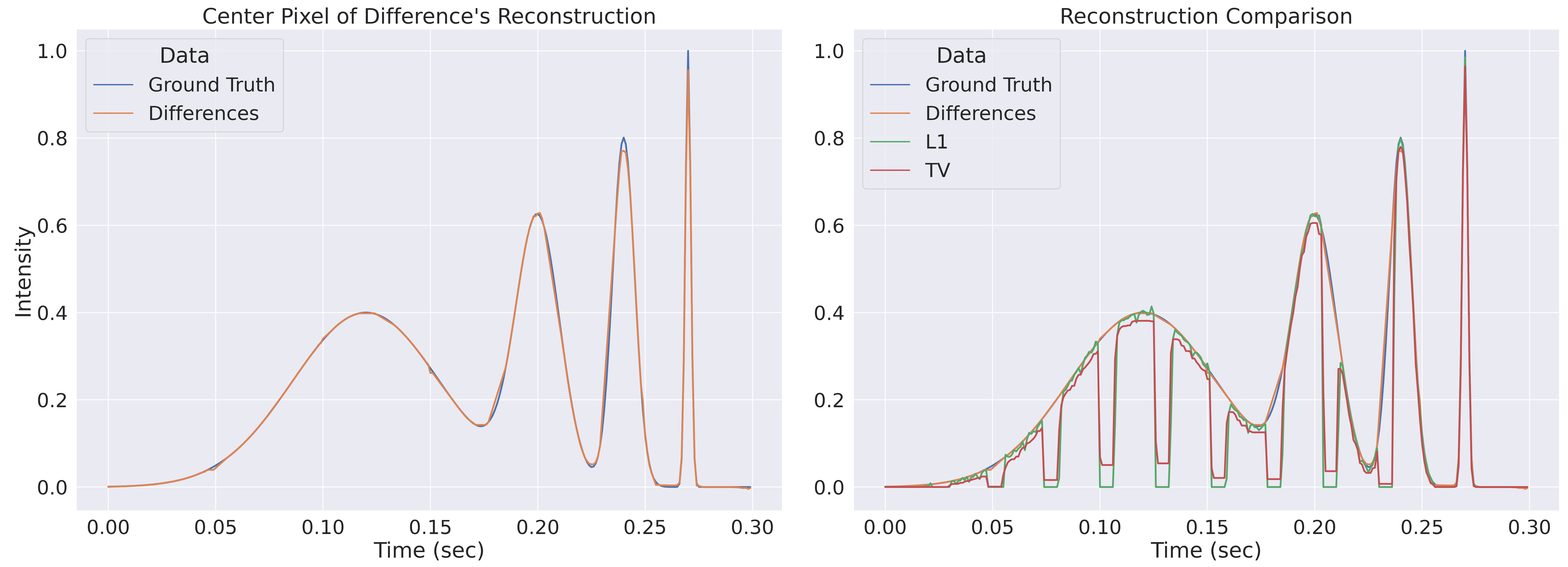}
    \end{tabular}
    \end{center}
    \caption{    \label{fig:ComparisonAlgo}
    Comparison of three compressed sensing algorithms: blocked differences ($B = 50$), a TV variant, and a standard $\ell^1$ variant. We plot the center pixel intensity of each reconstruction, where the PSTE is spatially localized. The two alternative algorithms suffer from a dropout artifact which is not noticeably present in our algorithm. Our algorithm ($\SI{31.1}{\second}$) was significantly faster than both the $\ell^1$ algorithm ($\SI{64.4}{\second}$) and TV algorithm ($\SI{190.8}{\second}$).}
\end{figure} 

\subsection{Validation of Theorem \ref{thm:avgErrDiff} in Simulation}\label{sec:ExpValidTheory}
Using the same simulated PSTE (Figure \ref{fig:SimulatedPSTE}) and optimization hyperparameters as the previous section, we looked at the effects of three sensor characteristics: the number of lines per sample, the sampling rate, and the power of the noise. Figure \ref{fig:TheoryValid} shows the results of our simulations using the blocked differences algorithm ($B=50$).

In the lines experiment, we enlarged the PSTE spatially by a factor of 12 in order to highlight its effect on the maximum feasible support size of a signal. Notably, there is a sharp decline in the error as the number of lines increases from one to seven, with smaller improvements thereafter. We conjecture that this is due to a phase transition in the feasible support size, and similar phenomena have been observed in other compressed sensing settings\cite{DonohoPhaseTrans}. From Theorem \ref{thm:RIPpsf}, the size of the maximum feasible support is roughly proportional to the number of lines sampled, and there is a critical number of lines $m$ where the corresponding maximum feasible support $k(m)$ shrinks below the size of the PSTE. Once this happens, the near-sparsity term of equation \ref{bnd:ReconError} spikes as it starts incorporating parts of the signal, leading to a sharp increase in error as the number of lines decreases.

In the sampling rate experiment, we fixed the number of lines per sample to 5 and only increased the sampling rate. There is small bump in the error as we increase the sampling rate, and we conjecture that this is due to a rolling shutter artifact (see Section \ref{sec:PractConsid}). Nevertheless, increasing the sampling rate generally decreases the average error as predicted.

Finally, in the noise experiment we corrupted each frame with i.i.d mean-zero Gaussian noise. The noise variance was calibrated to satisfy each of the tested signal-to-noise ratios (SNRs). The trend is roughly logarithmic as predicted, accounting for the fact that we measure the SNR in decibels.  

\begin{figure} [ht]
    \begin{center}
    \begin{tabular}{c} 
    \includegraphics[width=17cm]{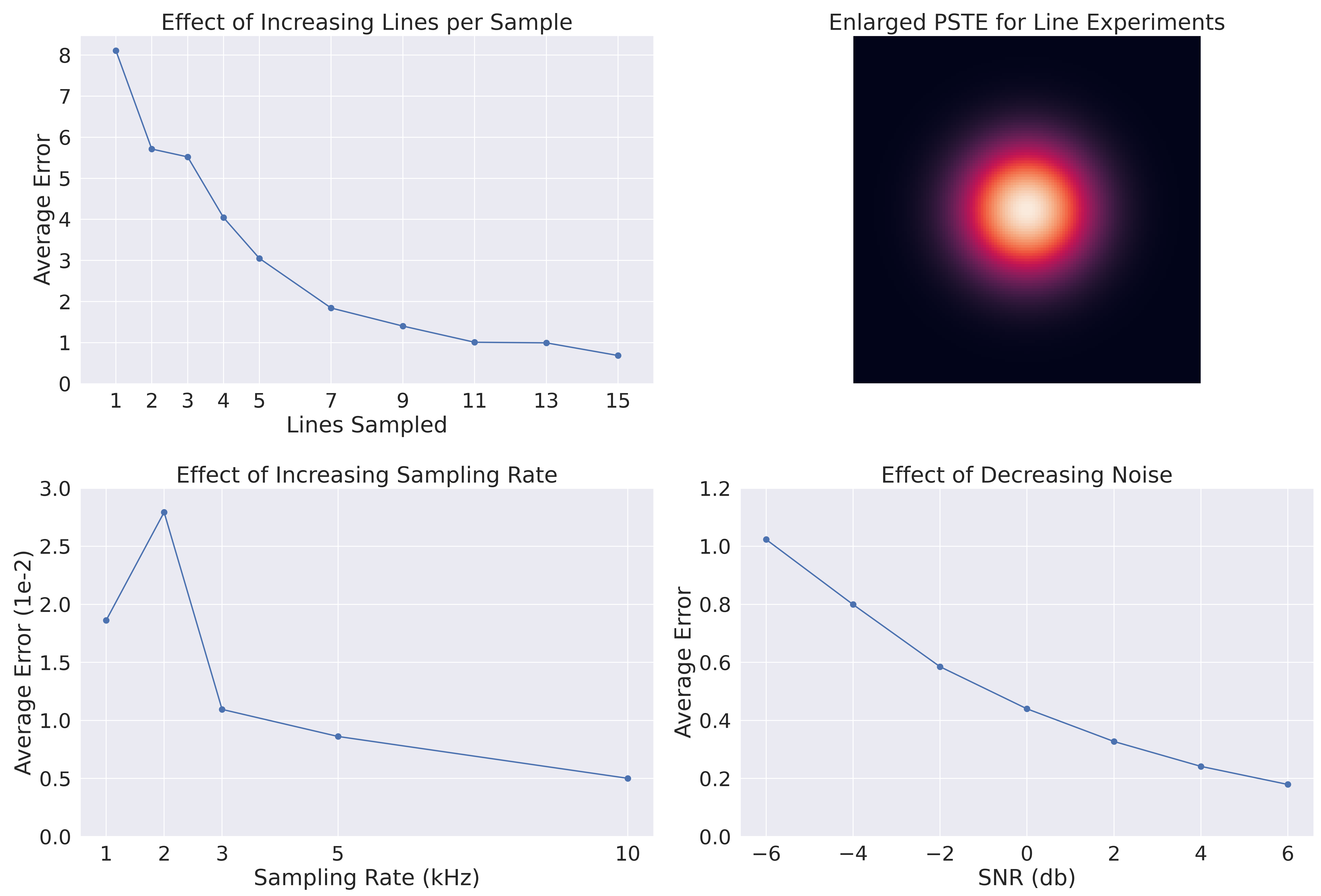}
    \end{tabular}
    \end{center}
    \caption{    \label{fig:TheoryValid}
Simulations demonstrating the predicted effects of each sensor characteristic on the average frame-wise error from Theorem \ref{thm:avgErrDiff}. For the lines experiment, we spatially enlarged the PSTE to demonstrate the relationship between the number of lines and the maximum spatial support of the signal.}
\end{figure} 

\subsection{Reconstruction up to the Nyquist Limit}\label{sec:BlipExperiments}
We test the maximum signal frequency that our blocked differences algorithm can accurately reconstruct from the measurements of the rolling shutter readout. Shown in Figure \ref{fig:BlipVis}, we generated a 1D signal consisting of pulses of varying frequencies occurring at regular intervals. The pulse frequencies linearly sweep from $\SI{40}{\hertz}$ to $\SI{500}{\hertz}$, the Nyquist limit of our simulated $\SI{1000}{\hertz}$ rolling shutter sampler. This was then multiplied by the same Gaussian of Figure \ref{fig:SimulatedPSTE} to generate the full signal.

We ran our blocked differences algorithm ($B = 50$) on this signal and looked at reconstruction error as function of pulse frequency. For each pulse, we computed the average frame-wise error over its temporal support and normalized by its power, as the pulses have varying power. The results of this simulation are shown in Figure \ref{fig:BlipError}. Even at frequencies up to the Nyquist limit of $\SI{500}{\hertz}$, we see our differences algorithm can generate high quality reconstructions from our rolling shutter readout. However, there is also a significant periodic trend in the error. This is an artifact due to spatial coverage  and the rolling shutter, and we discuss this in more detail as well as potential fixes in Section \ref{sec:PractConsid}. While this artifact presents a significant  challenge, these preliminary simulation results are encouraging and show that our differences algorithm is capable of reconstruction up to the Nyquist limit in a rolling shutter system.

\begin{figure} [ht]
    \begin{center}
    \begin{tabular}{c} 
    \includegraphics[width=16cm]{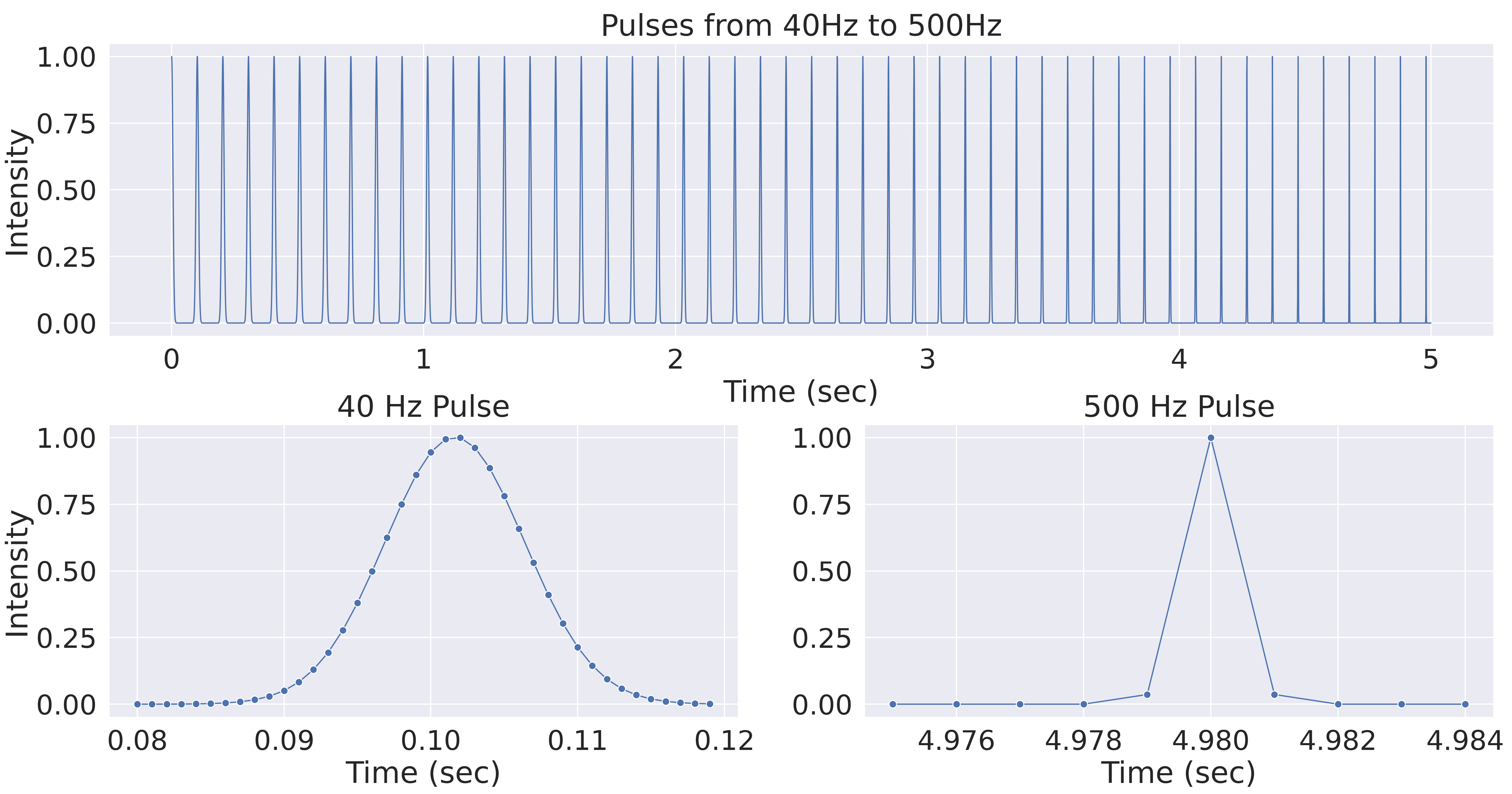}
    \end{tabular}
    \end{center}
    \caption{    \label{fig:BlipVis}
Signal consisting of pulses of varying frequency, linearly sweeping  from $\SI{40}{\hertz}$ to the Nyquist limit of $\SI{500}{\hertz}$. On the bottom, we zoom-in and visualize two pulses at the extreme ends of the signal.}
\end{figure}

\begin{figure} [hb]
    \begin{center}
    \begin{tabular}{c} 
    \includegraphics[width=15cm]{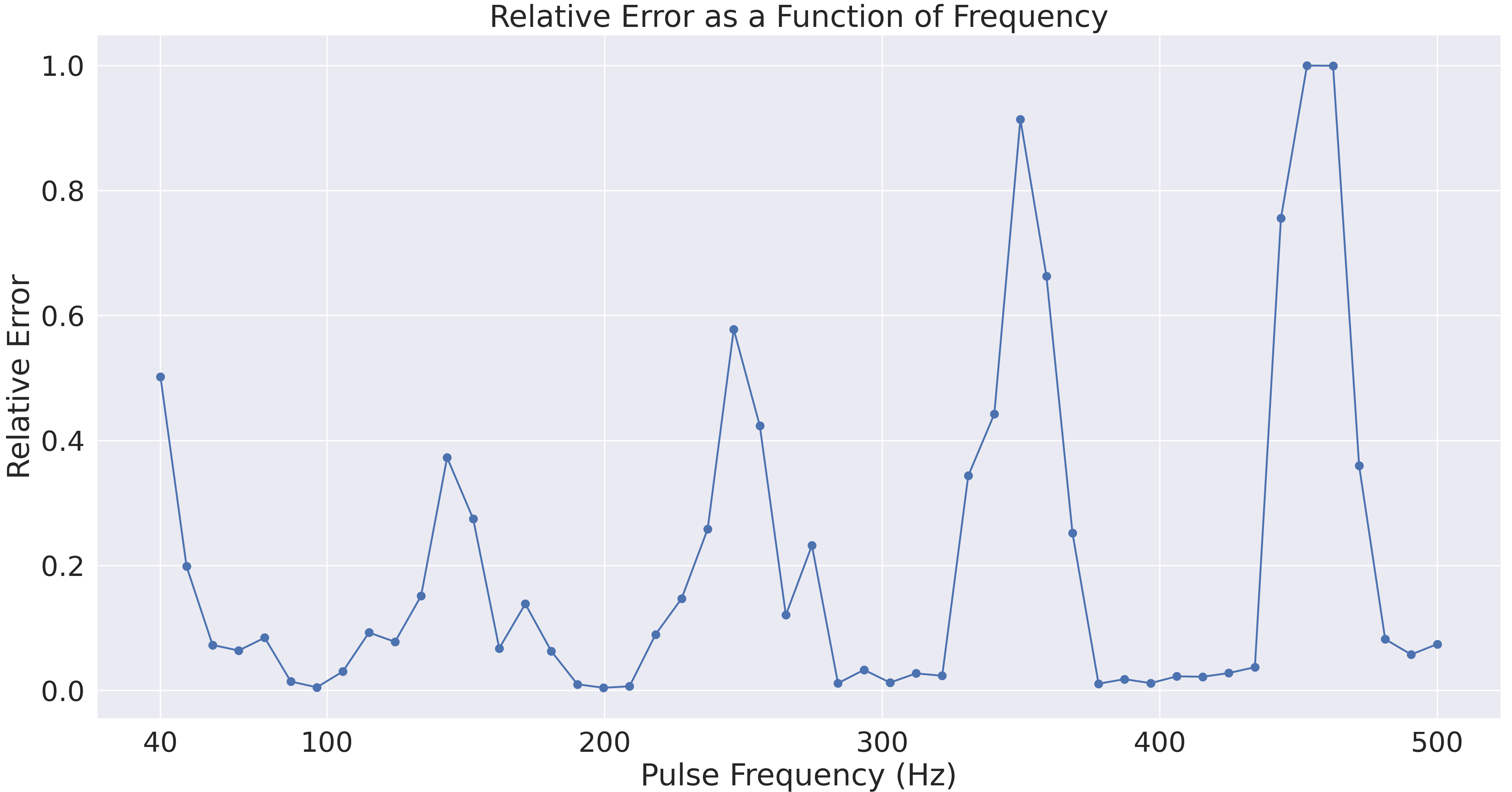}
    \end{tabular}
    \end{center}
    \caption{    \label{fig:BlipError}
Reconstruction error as a function of pulse frequency. We normalized the error by dividing by the pulses' power, as higher frequency pulses have smaller power due to their smaller temporal support.}
\end{figure} 

\clearpage

\section{Practical Considerations}\label{sec:PractConsid}
In Figure \ref{fig:ComparisonAlgo}, we observed a dropout phenomenon in the reconstructions of the  TV and $\ell^1$ algorithms, where the reconstructed signal would periodically go to zero. Similarly, in Figure \ref{fig:BlipError}, we observed a periodic behavior in the reconstruction error. Both of these phenomena can be attributed to aliasing artifacts arising from a lack of sufficient spatial coverage in the rolling shutter sampler. 

\subsection{An Explanation of the Artifact}
Convolution, as defined mathematically, requires a wrap-around behavior: $(\zeta \ast x) (j) = \sum_{j} \zeta(j) x([i-j]_n)$ where $[i-j] \equiv (i -j) \mod n$. However, in a physical imaging system the convolution of a signal with the PSF is cut off at the boundaries of the FPA (we replicate this effect in simulation via padding and cropping operations). Due to this boundary artifact, there is spatial bias in a pixel's location relative to the signal, as a signal tends to spread more power to nearby pixels. Furthermore, the phase diffuser PSF we used in our simulations is spatially limited and does not spread a point source over the entire FPA (see Figure \ref{fig:OpticalSystem}). For PSTEs localized at the center of the FPA, these two factors lead to essentially dead lines of pixels at the top and bottom of the FPA that do not respond to the PSTE. For example, the bottom row of Figure \ref{fig:MeasurementSeq} shows the rolling shutter sampling near the top in the first frame ($t=0$), missing the diffused signal. If the rolling shutter samples these dead lines when the PSTE is spiking, we get no information during that critical spiking window and reconstruction is impossible. 

Hence, the periodic dropout observed  in Figure \ref{fig:ComparisonAlgo} can be explained by the rolling shutter sampling dead lines at the top and bottom of the FPA, inducing an aliasing effect from these uninformative measurements. Similarly, the periodicity of the reconstruction error in Figure \ref{fig:BlipError} can be explained by the signal (Figure \ref{fig:BlipVis}) and rolling shutter coming in and out of phase.  When a pulse is spiking, the rolling shutter samples the central lines if the two are in phase and the dead lines at the edges if the two are out of phase. Since accurate reconstruction depends on good samples during a pulse's spiking window, reconstruction error improves as the sampled lines get closer to the center.

While Theorem \ref{thm:RIPpsf} generally gives a good characterization of our imaging system, it does not account for the boundary artifacts inherent to a physical FPA. This is a fundamental obstacle between theory and practice. Similarly, the phase diffuser PSF we used does not exactly match the requirements of Theorem \ref{thm:RIPpsf}, and physically implementing a PSF that more closely matches these requirements is one of our current areas of research.

\subsection{Increasing Spatial Coverage through a Double Shutter System}
The central problem is a lack of adequate spatial coverage in the rolling shutter when dealing with spatially limited signals like PSTEs. One hardware solution we are investigating is a double rolling shutter system. Such a system samples two lines at a time, with the lines spaced apart by half of the FPA's height. Due to this spacing, if one line is near the edge then the other is near the center, potentially solving the spatial coverage problem.

We compare reconstruction quality under simulated double shutter and single shutter systems with comparable sampling budgets. The single shutter system samples four lines at time, while the double shutter system samples two sets of two lines at a time spaced apart by half the FPA's height. The other parameters of both systems match that of the single shutter system in Section \ref{sec:Experiments}. We again used the PSTE of Figure \ref{fig:SimulatedPSTE} as the input signal. To account for phase effects between the shutter and signal, we shifted the rolling shutter schedule by increments of $\SI{3}{\milli\second}$ and tested reconstruction for each shift using the blocked differences algorithm ($B = 50$). Figure \ref{fig:SglShutter} shows single shutter system suffers from spatial coverage issues, as there is a phase transition in recovery of the fast $\SI{400}{\hertz}$ pulse depending on whether the rolling shutter is near the FPA edges during that pulse. In contrast, Figure \ref{fig:DblShutter} demonstrates that the double shutter system provides sufficient spatial coverage, as reconstruction is stable across all phase offsets between the rolling shutter and the signal. Even though both systems sample the same number of lines at a time, the double shutter system has better spatial coverage because of its line spacing and avoids the artifact.

\begin{figure} [ht]
    \begin{center}
    \begin{tabular}{c} 
    \includegraphics[width=17cm]{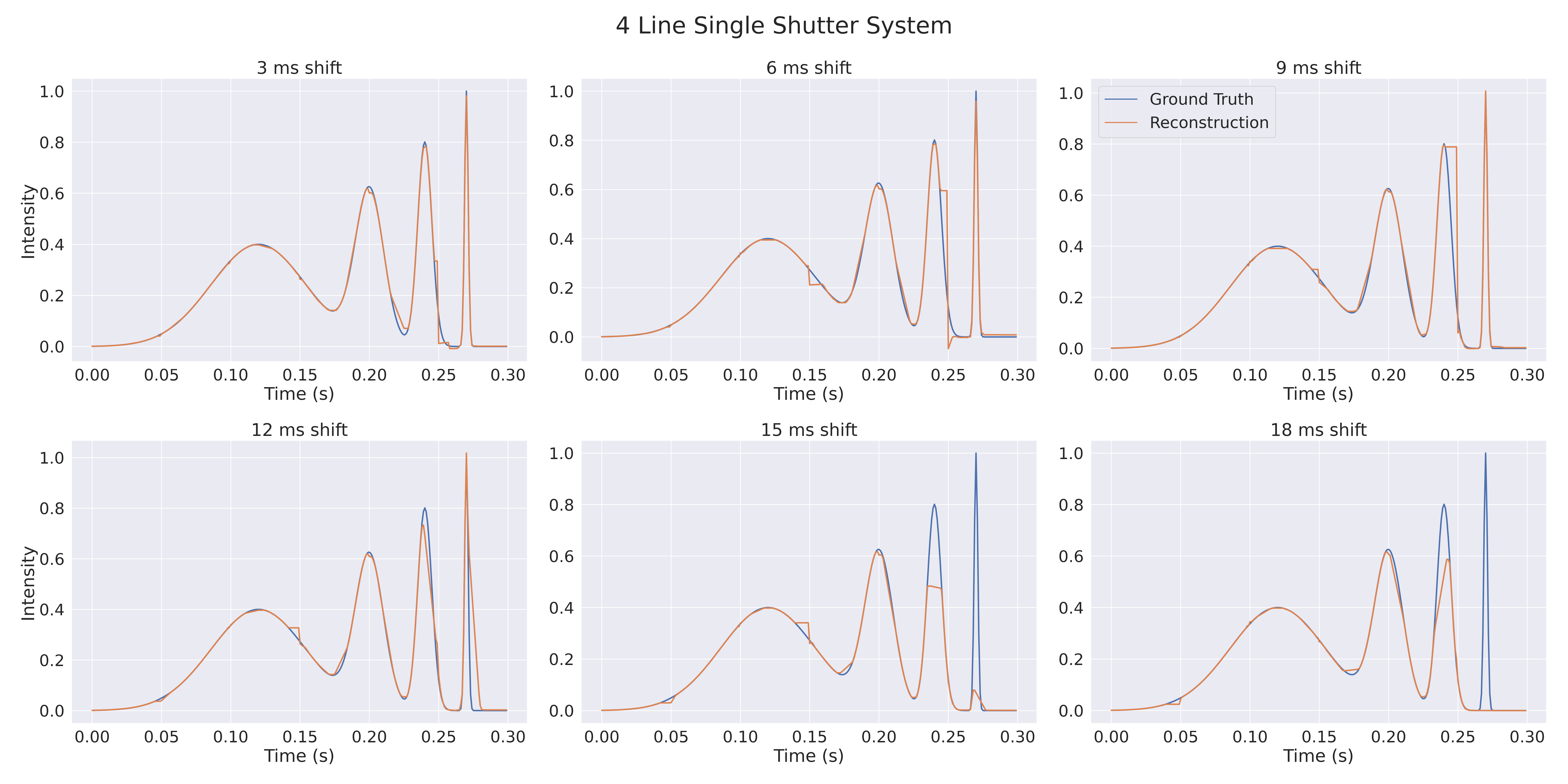}
    \end{tabular}
    \end{center}
    \caption{    \label{fig:SglShutter}
Demonstration of the spatial coverage artifact in a single shutter system sampling 4 lines at $\SI{1000}{\hertz}$. We incremented the rolling shutter schedule to show the phasing effect between the rolling shutter and the signal, where the sampled lines can be closer or farther to the signal's location at the time of spiking. There is a phase transition in recovery of the $\SI{400}{\hertz}$ pulse depending on whether the rolling shutter samples lines at the edge of the FPA.}
\end{figure} 

\begin{figure} [ht]
    \begin{center}
    \begin{tabular}{c} 
    \includegraphics[width=17cm]{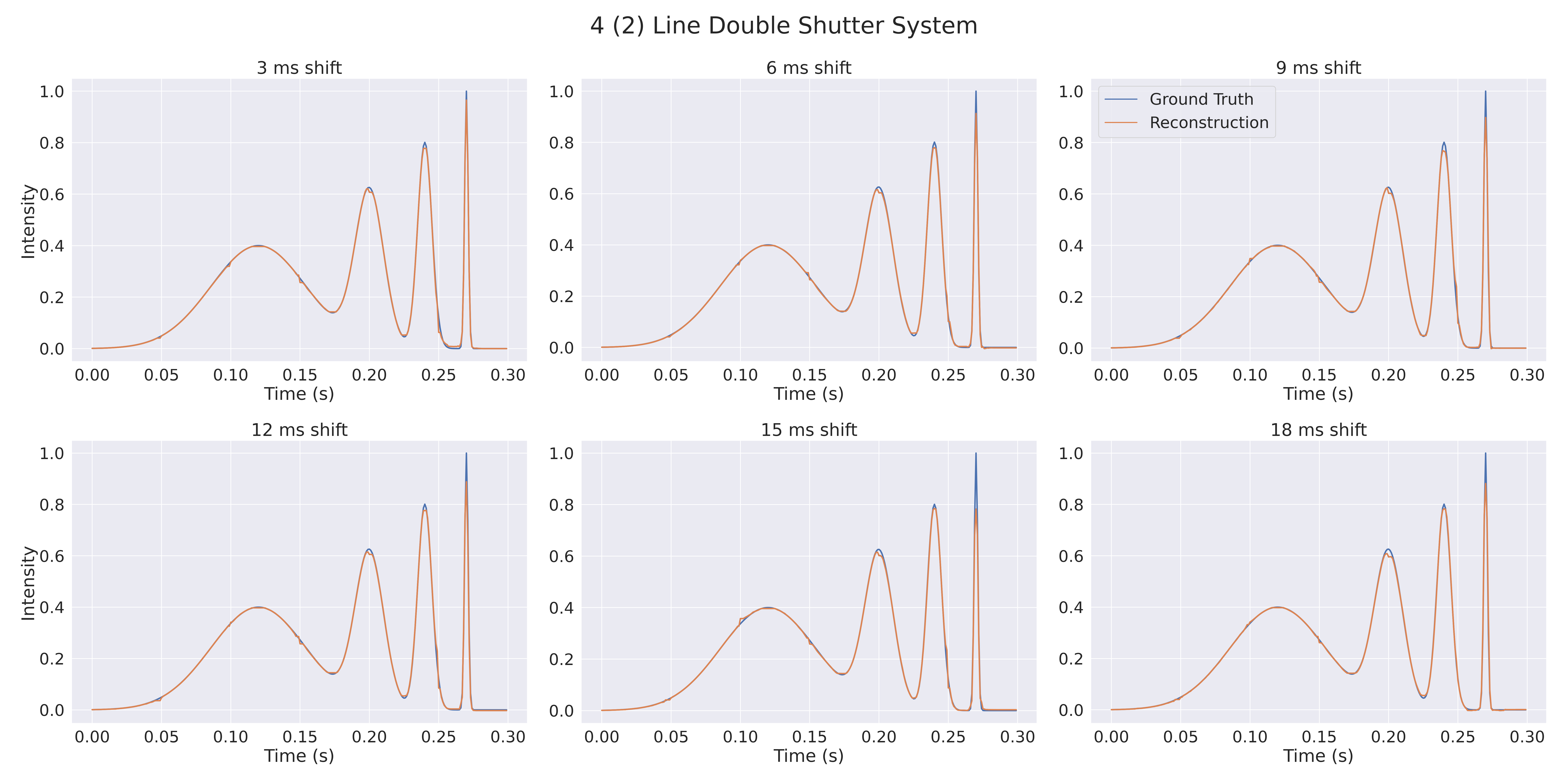}
    \end{tabular}
    \end{center}
    \caption{    \label{fig:DblShutter}
A double shutter system sampling 4 total lines (2 lines per shutter) at the same rate of $\SI{1000}{\hertz}$. Unlike the single shutter system, there is no spatial aliasing artifact. Due to the spacing between the sampled lines, the double shutter system provides sufficient spatial coverage to reliably recover extremely fast and small signals.}
\end{figure} 

\clearpage 

\subsection{Spatial Coverage, Signal Frequency, and Sampling Rate}
In Figure \ref{fig:SglShutter}, we see that the spatial coverage artifact primarily affects the $\SI{400}{\hertz}$ pulse, whereas reconstruction of the slower pulses is mostly stable across all shifts. Again, this artifact occurs if the rolling shutter samples dead lines when a pulses spikes, since we get no information during that critical window. While this window is very brief for fast pulses (around 1-2 samples for a $\SI{400}{\hertz}$ pulse), the spike of a slower pulse occurs over a longer period. If the rolling shutter sampler samples an informative set of lines during this spike, reconstruction can happen, and there are effectively more chances to sample the spike of a slower pulse. Alternatively, the spike of a slower pulse has a larger spatial footprint in the number of informative lines. Thus, one can also solve the spatial coverage artifact by having  a sufficiently fast shutter relative to the signal's maximum frequency. Figure \ref{fig:4xSglShutter} shows the result of increasing the four line single shutter system's sampling rate to $\SI{4000}{\hertz}$. We see that the artifact is largely gone, and reconstruction is stable across all offsets between the shutter and signal.

\begin{figure} [ht]
    \begin{center}
    \begin{tabular}{c} 
    \includegraphics[width=17cm]{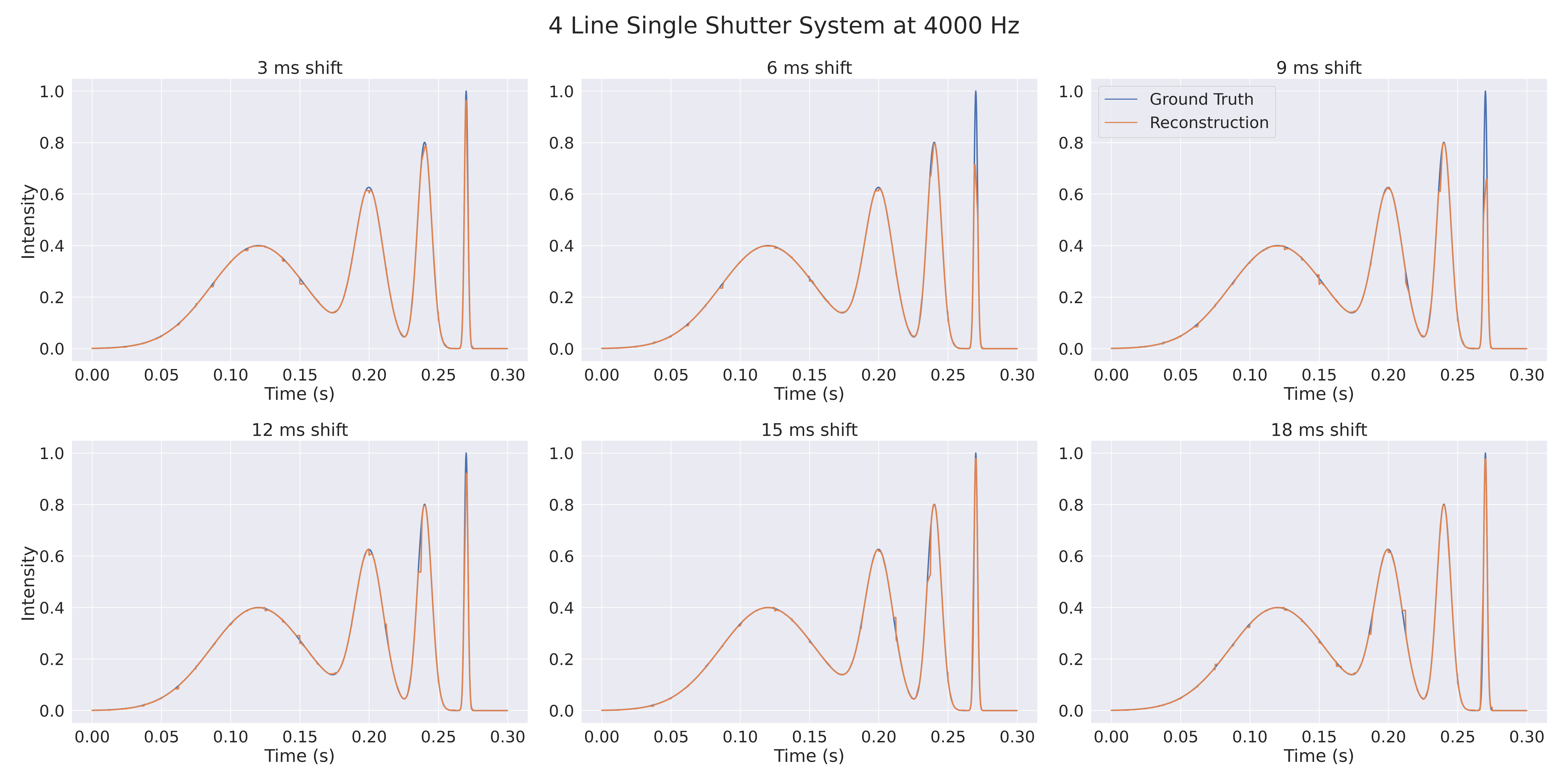}
    \end{tabular}
    \end{center}
    \caption{    \label{fig:4xSglShutter}
We demonstrate the effect of speeding up the singe rolling  shutter's sampling rate from $\SI{1000}{\hertz}$ to $\SI{4000}{\hertz}$. Reconstruction is stable across all offsets of the rolling shutter schedule.}
\end{figure} 

\section{Conclusion}
We have the demonstrated the efficacy of our blocked differences algorithm (Algorithm \ref{alg:BlockedFISTA-D}) in reconstructing PSTEs from the rolling shutter readout of a camera, accurately recovering signals orders of magnitude faster than the native global shutter rate and up to the Nyquist limit of the rolling shutter sampling rate. Compared to alternative TV and $\ell^1$ norm algorithms, our algorithm is both faster and offers superior reconstruction quality (Figure \ref{fig:ComparisonAlgo}). Our theoretical results characterize how certain parameters of our imaging system - the rolling shutter rate, the time integration window, the power of exogenous noise - affect our algorithm's reconstruction error (Section \ref{sec:Theory}). These theoretical results were validated in simulation (Figure \ref{fig:TheoryValid}), and they inform how a physical rolling shutter system should be tuned to accommodate a signal of interest. 

We also identified an aliasing artifact arising from spatial coverage issues in the rolling shutter readout (Section \ref{sec:PractConsid}). While this is an important barrier in a rolling shutter imaging system, we demonstrated the efficacy of two potential hardware solutions: utilizing a double shutter or using a faster camera (Figures \ref{fig:DblShutter} and \ref{fig:4xSglShutter}). Complementing the algorithmic work presented in this paper, we are currently testing the capabilities of a coded-aperture optical system we developed. This system has the potential to implement in hardware PSFs that more closely match the assumptions of Theorem \ref{thm:RIPpsf}, improving reconstruction quality and providing another potential solution to the spatial coverage problem. This system also performs phase retrieval of the PSF, a critical component in our algorithmic pipeline. Finally, it can implement and record various signals, allowing us to validate in hardware the simulation results presented in this paper. Our next step is using this coded aperture system to bridge the gap between the algorithmic work presented in this paper and an implementation in hardware.

\appendix    

\section{Optimization Details}
\label{app:OptDetails}
Recall that our algorithm seeks to recover the vector of frame-by-frame differences $d \in \R^{P \times T}$ of a movie $x \in \R^{P \times T}$ from a sequence of measurements $y \in \R^{M \times T}$. The sensing matrices $A^{(t)} = P_r^{(t)} H$ can be written as the composition of a projection operator $P_r^{(t)}$ representing our row sampler at time $t$ and a circulant matrix $H$ representing convolution with our PSF. After getting the difference estimates $\hat{d}$, we can reconstruct each frame by adding up the differences: $\hat{x}^{(t)} = \sum\limits_{s \leq t} \hat{d}^{(s)}$. To generate the difference estimates, we minimize the objective of equation \ref{eqn:DiffObj}:
\[
\sum_{t = 0}^{T-1} \frac{1}{2}\| A^{(t)}(\sum_{s \leq t} d^{(s)}) - y^{(t)} \|^2_2 + \lambda \| d^{(t)}] \|_1 
\]
Since our optimization problem is a $\ell^2 - \ell^1$ minimization problem, we use the FISTA\cite{FISTA} optimizer which has the general update rule:
\[
x_{m+1} = soft_{\lambda}[x_{m} - \tau \partial x_m] 
\]
where $\tau$ is the stepsize, $\partial x_m$ is the gradient of the $\ell^2$ term with respect to $x_m$, and $soft$ is the soft-thresholding operator.

\subsection{$\ell^2$ Gradient Derivation}\label{app:GradDeriv}
We derive the gradient of the $\ell^2$ term in the objective function. Let $\eta^{(t)} = A^{(t)}[\sum\limits_{s \leq t} d^{(s)} - x^{(t)}]$ denote the error between the measurement under $\hat{x}^{(t)}$ and the observed $y^{(t)}$. Note that each difference only affects the reconstruction for the current and subsequent frames. We compute: 
\begin{align*}
	\pdv{}{d^{(i)}}  \sum_{t = 0}^{T-1} \frac{1}{2}  \| A^{(t)}(\sum_{s \leq t} d^{(s)}) - y^{(t)} \|^2_2 &= \pdv{}{d^{(i)}}  \sum_{t \geq i} \frac{1}{2} \| A^{(t)}(\sum_{s \leq t} d^{(s)}) -y^{(t)} \|^2_2\\
	&= \sum_{t \geq i} A^{(t) T}  [A^{(t)}(\sum_{s \leq t} d^{(s)}) - y^{(t)}] \\
	&=  \sum_{t \geq i}  H^T P_r^{(t) T}  [P_r^{(t)} H (\sum_{s \leq t} d^{(s)}) - y^{(t)}] \\
	&=  \sum_{t \geq i}  H^T[P_r^{(t)} H (\sum_{s \leq t} d^{(s)}) - y^{(t)}] \\
	&=  \sum_{t \geq i} H^T \eta^{(t)}\\
	&= H^T \sum_{t \geq i} \eta^{(t)}
\end{align*}
In the fourth equality, we used the fact that $P_r^{(t)}$ is a projection matrix and hence is both symmetric and idempotent. Thus, for the $i^{th}$ difference $d^{(i)}$ our FISTA update rule:
\[
d^{(i)}_{m+1} = soft_{\lambda}[d^{(i)}_{m} - \tau H^T \sum_{t \geq i} \eta_m^{(t)}] 
\]
From the above equation, our inner gradient step can be expressed as:
\begin{equation}\label{eqn:FISTA_innergrad}
d_m - \tau H^T S^T \eta_m = d_m - \tau H^T S^T (P H S d_m)
\end{equation}
$S$ is an lower triangular matrix of 1's representing a summation operator, since at time $t$ the reconstruction is the sum of all prior differences. $P$ is a projection matrix representing our row sampler. $H$, by slight abuse of notation, is a block-diagonal matrix representing convolving the PSF with the entire sequence of differences $d$.

\subsection{Calibrating the Stepsize}\label{app:StepsizeCalc}
FISTA relies on each inner gradient step being a contraction, so the stepsize $\tau$ needs to be calibrated to ensure this\cite{FISTA}.  From equation \ref{eqn:FISTA_innergrad}, this amounts to requiring that $\tau$ satisfy the inequality:
\[
\tau \| H^T S P H S \|_2 \leq 1 
\]
where $\| \cdot \|_2$ denotes the spectral norm.

The circulant matrix $H$ represents convolution with the PSF and can be diagonalized in the Fourier basis: $H = \mathcal{F}^{-1} \Lambda \mathcal{F}$ with the diagonal of $\Lambda$ coinciding with the Fourier transform of the PSF. Therefore, if we let $\zeta$ denote the PSF, we have $\| H \|_2 = \max_i(\abs{\mathcal{F} \zeta}_i) = \| \mathcal{F} \zeta \|_\infty$. Similarly, we compute the Singular Value Decomposition (SVD) of our summation matrix $S$ and recover the first singular value $\sigma_1(S) = \| S \|_2$. As $P$ is a projection matrix, we have $\| P \|_2 = 1$. Therefore, we set
\[
\tau = (\| \mathcal{F} \zeta \|_\infty \sigma_1(S) r)^{-2}
\]
where $r$ is the number of passes the rolling shutter makes over the FPA. This ensures that our gradient step is a contraction. 

This stepsize is very conservative and uses the loose bound of $\| P H \|_2 \leq \| P \|_2 \| H \|_2$. Empirically, we found that relaxing the scaling of $\tau$ with respect to the summation operator yields faster convergence without sacrificing performance:
\[
\tau_{emp} = \| \mathcal{F} \zeta \|_\infty^{-2} (\sigma_1(S) r)^{-1}
\]

\section{Theoretical Supplement}\label{app:TheoryResults}

\subsection{Proof of Theorem \ref{thm:avgErrDiff}}
\begin{lemma}\label{lem:L1boundDiff}
For $x \in \R^n$, let $\nabla x \in \R^{n-1}$ be the vector of forward differences: $\nabla x (i) = x(i + 1) - x(i)$. We have the bound:
\begin{equation}\label{eqn:L1boundDiff}
    \| x \|_1 \leq n (\| x(0) \|_1 + \| \nabla x \|_1)
\end{equation}
\begin{proof}
Each entry of $x$ can be written as the sum:
\[
|x(i)| = |\sum_{j=0}^{i-1} \nabla x (j) + x(0)| \leq | \sum_{j=0}^{i-1} \nabla x (j) | + |x(0)| \leq \| \nabla x \|_1 + | x(0)|
\]
Summing over entries of $x$ gives the result.
\end{proof}
\end{lemma}

\begin{proof}
For notational convenience, we drop the subscript on the forward time difference operator $\nabla \equiv \nabla_t $. Similarly, we abuse notation and let $u_{\Bar{I}}$ denote the appropriately subset $u$ for time-slices and movies alike. We first derive a cone condition that characterizes the solution $\hat{x}$. As the ground truth signal $x^*$ is feasible, by optimality of $\hat{x}$ we have:
\begin{equation}\label{calc:cone1}
\|\hat{x}^{(0)}\|_1 + \norm{\nabla \hat{x}}_1 \leq \| x^{*(0)}\|_1 + \norm{\nabla x^*}_1
\end{equation}
Letting $h = \hat{x} - x^*$ denote the error, we calculate:
\begin{align*}
    \| \hat{x}^{(0)}\|_1 + \norm{\nabla \hat{x}}_1 &= \| x^{*(0)} + h^{(0)} \|_1 + \norm{\nabla x^* + \nabla h}_1\\
    &= \| [x^{*(0)} + h^{(0)}]_{\Bar{I}} \|_1 + \| [x^{*(0)} + h^{(0)}]_{\Bar{I}^c} \|_1 + \norm{[\nabla x^* + \nabla h]_{\Bar{I}}}_1 +  \norm{[\nabla x^* + \nabla h]_{\Bar{I}^c}}_1\\
    &\geq \| x^{*(0)}_{\Bar{I}} \| - \| h^{(0)}_{\Bar{I}} \| + \| h^{(0)}_{\Bar{I}^c} \| - \| x^{*(0)}_{\Bar{I}^c} \|+ \norm{\nabla x^*_{\Bar{I}}}_1 - \norm{\nabla h_{\Bar{I}}}_1 + \norm{\nabla h_{\Bar{I}^c}}_1 - \norm{\nabla x^*_{\Bar{I}^c}}_1 \numberthis \label{calc:cone2}
\end{align*}
Combining inequalities \eqref{calc:cone1} and \eqref{calc:cone2} gives the general cone condition:
\begin{equation}\label{eqn:gradConegeneral}
\| h^{(0)}_{\Bar{I}^c} \|_1 + \norm{\nabla h_{\Bar{I}^c}}_1 \leq \| h^{(0)}_{\Bar{I}} \|_1 + \norm{\nabla h_{\Bar{I}}}_1 + 2[\| x^{*(0)}_{\Bar{I}^c} \|_1 + \norm{\nabla x^*_{\Bar{I}^c}}_1]
\end{equation}
If the $k$-sparse approximations have the same support $I$, equation \eqref{eqn:gradConegeneral} can be rewritten:
\begin{equation}\label{eqn:gradCone}
\| h^{(0)}_{\Bar{I}^c} \|_1 + \norm{\nabla h_{\Bar{I}^c}}_1 \leq \| h^{(0)}_{\Bar{I}} \|_1  + \norm{\nabla h_{\Bar{I}}}_1 + 2 [\sigma_k(x^{*(0)}) + \sigma_k(\nabla x^*)]
\end{equation}
For notational convenience we use the cone condition of equation  \eqref{eqn:gradCone}, and the general case follows by replacing it with equation \eqref{eqn:gradConegeneral}.

Let $J(t) \subseteq I^c$ denote the indices of the largest $k$ magnitude entries of $h^{(t)}$ outside of $I$. As our sensing matrices $A^{(t)}$ are uniformly RIP with constant $\delta_{2k}$  and $\| A^{(t)} h^{(t)} \|_2 \leq \epsilon$, we have a standard compressed sensing bound for each frame (Lemma 1.3 \cite{Eldar_Kutyniok_2012}):
\begin{align*}
(1-\delta_{2k})\| h^{(t)}_{\Bar{I} \cup J(t)} \|_2 &\leq \delta_{2k} \sqrt{2} \frac{ \| h^{(t)}_{\Bar{I} ^c} \|_1}{\sqrt{k}} + \frac{| \dprod{Ah_{\Bar{I} \cup J(t)}}{ A h} |}{\| h_{\Bar{I} \cup J(t)} \|_2}\\
&\leq \delta_{2k} \sqrt{2} \frac{ \| h^{(t)}_{\Bar{I} ^c} \|_1}{\sqrt{k}} + \frac{\| A  h_{\Bar{I} \cup J(t)}\|_2 \| A h \|_2}{\| h_{\Bar{I} \cup J(t)} \|_2}\\
&\leq \delta_{2k} \sqrt{2} \frac{ \| h^{(t)}_{\Bar{I} ^c} \|_1}{\sqrt{k}} + (1+ \delta_{2k})^{1/2} \epsilon
\end{align*}
Adding up the frame-wise bounds, an application of Lemma \ref{lem:L1boundDiff} and the cone condition \eqref{eqn:gradCone} gives:
\begin{align*}
    (1-\delta_{2k}) \sum_{t} \| h^{(t)}_{\Bar{I} \cup J(t)} \|_2 &\leq \delta_{2k} \sqrt{2} \frac{ \| h_{\Bar{I}^c} \|_1}{\sqrt{k}} + (1+ \delta_{2k})^{1/2} T \epsilon \\
    &\leq \sqrt{2} T ( \delta_{2k}  \frac{\| h^{(0)}_{\Bar{I}^c} \|_1 + \| \nabla h_{\Bar{I}^c} \|_1}{\sqrt{k}} + \epsilon)\\
    &\leq \sqrt{2} T ( \delta_{2k} \frac{\| h^{(0)}_{\Bar{I}} \|_1 + \| \nabla h_{\Bar{I}} \|_1}{\sqrt{k}} +  \delta_{2k} \frac{2 [\sigma_k(x^{*(0)}) + \sigma_k(\nabla x^*)]}{\sqrt{k}} + \epsilon) \numberthis \label{calc:bound1}
\end{align*}
We now bound the first term:
\begin{align*}
    \frac{ \| h^{(0)}_{\Bar{I}} \|_1  + \norm{\nabla h_{\Bar{I}}}_1}{\sqrt{k}} &= \frac{1}{\sqrt{k}}( \| h^{(0)}_{\Bar{I}} \|_1 + \sum\limits_{t = 1}^{T-1} \| h_{\Bar{I}}^{(t)} - h_{\Bar{I}}^{(t-1)} \|_1) \\
    &\leq \frac{2}{\sqrt{k}}(\sum\limits_{t = 0}^{T-1} \| h_{\Bar{I}}^{(t)} \|_1)\\
    &\leq 2  \sum_{t} \| h^{(t)}_{\Bar{I} \cup J(t)} \|_2
\end{align*}
Plugging this into \eqref{calc:bound1} and multiplying both sides by $(1-\delta_{2k})^{-1} < \sqrt{2}$:
\[
\sum_{t} \| h^{(t)}_{\Bar{I} \cup J(t)} \|_2 \leq  \delta_{2k} 4T [\| \sum_{t} \| h^{(t)}_{\Bar{I} \cup J(t)} \|_2 +  \frac{ \sigma_k(x^{*(0)}) + \sigma_k(\nabla x^*)}{\sqrt{k}}] + 2 T\epsilon
\]
As $\delta_{2k} < (4T)^{-1}$, we simplify to get:
\begin{equation}\label{calc:firstPartBound}
\sum_{t} \| h^{(t)}_{\Bar{I} \cup J(t)} \|_2 \leq C'' (\frac{\sigma_k(x^{*(0)}) + \sigma_k(\nabla x^*)}{\sqrt{k}} + 2 T \epsilon)
\end{equation}
where $C'' = (1 - \delta_{2k} 4T)^{-1}$.  Now, we bound $\sum\limits_{t} \| h^{(t)}_{[\Bar{I} \cup J(t)]^c} \|_2$. We start with a standard bound for each frame (Lemma A.4 \cite{Eldar_Kutyniok_2012}):
\[
\norm{h^{(t)}_{(\Bar{I} \cup J(t))^c}}_2 \leq \frac{\norm{h^{(t)}_{\Bar{I}^c}}_1}{\sqrt{k}} 
\]
Adding up these frame-wise bounds and using the same calculation as above gives:
\begin{align*}
\sum_{t} \| h^{(t)}_{(\Bar{I} \cup J(t))^c} \|_2 &\leq \frac{\norm{h_{\Bar{I}^c}}_1}{\sqrt{k}}\\
&\leq 2 T [\sum_{t} \| h^{(t)}_{\Bar{I} \cup J(t)} \|_2  + \frac{ \sigma_k(x^{*(0)}) + \sigma_k(\nabla x^*)}{\sqrt{k}}] \numberthis \label{calc:secondPartBound}
\end{align*}
Combining inequalities \eqref{calc:firstPartBound} and \eqref{calc:secondPartBound} finishes the result:
\begin{align*}
    \frac{1}{T} \sum_t \| h^{(t)} \|_2 &\leq \frac{1}{T} [\sum_{t} \| h^{(t)}_{\Bar{I} \cup J(t)} \|_2 + \| h^{(t)}_{(\Bar{I} \cup J(t))^c} \|_2] \\
    &\leq \frac{1}{T}[(2T + 1) \sum_{t} \| h^{(t)}_{\Bar{I} \cup J(t)} \|_2  + 2 T \frac{ \sigma_k(x^{*(0)}) + \sigma_k(\nabla x^*)}{\sqrt{k}}]\\
    &\leq C  \frac{ \sigma_k(x^{*(0)}) + \sigma_k(\nabla x^*)}{\sqrt{k}} + C' T \epsilon
\end{align*}
\end{proof}

\section{Comparison Algorithm Details} \label{app:AltAlgos}
In this section, we briefly provide details the alternative CS algorithms used in our benchmarks.

\subsection{Total Variation (TV) Algorithm}
Throughout this paper, we have been treating our movies $x \in \R^{P \times T}$ as 2D arrays by combining the two spatial dimensions: $P = N^2$. However, in reality each frame $x^{(t)}$ is a $N \times N$ array. For this section, we distinguish between the two spatial dimensions and treat all movies  at 3D arrays: $x \in \R^{N \times N \times T}$. For a $3$-dimensional array $x$, its total variation norm  is  the $\ell^1$ norm of the forward gradient in all three directions:
\[
\| x \|_{TV} = \sum_{i=1}^3 \| \nabla_i x \|_1
\]
where $\nabla_i x (v) = x(v + e_i) - x (v)$ with $v$ being an array index and $e_i$ being the $i^{th}$ coordinate vector. Then, the total variation algorithm reconstructs by solving the following objective:
\[
\hat{x}_{TV} = \argmin_x \sum_{t = 0}^{T-1} \frac{1}{2} \| A^{(t)} x^{(t)} - y^{(t)} \|_2^2 + \lambda \| x \|_{TV}
\]
The stepsize $\tau_{TV}$ was calibrated using a similar computation as in Appendix \ref{app:StepsizeCalc}:
\[
\tau_{TV} = \| \mathcal{F} \zeta \|_\infty^{-2}
\]
We also tuned the regularization parameter $\lambda$ to penalize the spatial and time gradients differently, as we found this gave better performance:
\[
\lambda_t = 2 \lambda \qquad \lambda_{xy} = .9 \lambda
\]

\subsection{Standard $\ell^1$ Algorithm}
The standard $\ell^1$ algorithm for a movie sequence is just the usual $\ell^1$ algorithm applied to each frame. Its reconstruction is the solution to the problem:
\[
\hat{x}_{\ell^1} = \argmin_x \sum_{t = 0}^{T-1} \frac{1}{2} \| A^{(t)} x^{(t)} - y^{(t)} \|_2^2 + \lambda \| x^{(t)} \|_1
\]
Note that each frame estimate $x^{(t)}$ doesn't affect other estimates, and jointly solving for all frames at once is the same as separately solving for each frame. We use the same calculation in the TV norm case to calibrate the stepsize:
\[
\tau_{\ell^1} = \| \mathcal{F} \zeta \|_\infty^{-2}
\]

.

\acknowledgments 
The authors would like to Richard B. Lehoucq (Sandia National Laboratories) and Oscar Lopez (Florida Atlantic University) for their helpful discussions about the theoretical aspects of compressed sensing. 

The work in this paper was funded by LDRD grant \#24-0112 of Sandia National Laboratories. Sandia National Laboratories is a multi-mission laboratory managed and operated by National Technology \& Engineering Solutions of Sandia, LLC (NTESS), a wholly owned subsidiary of Honeywell International Inc., for the U.S. Department of Energy’s National Nuclear Security Administration (DOE/NNSA) under contract DE-NA0003525. This written work is authored by an employee of NTESS. The employee, not NTESS, owns the right, title and interest in and to the written work and is responsible for its contents. Any subjective views or opinions that might be expressed in the written work do not necessarily represent the views of the U.S. Government. The publisher acknowledges that the U.S. Government retains a non-exclusive, paid-up, irrevocable, world-wide license to publish or reproduce the published form of this written work or allow others to do so, for U.S. Government purposes. The DOE will provide public access to results of federally sponsored research in accordance with the DOE Public Access Plan.

\bibliography{refs} 
\bibliographystyle{spiebib} 

\end{document}